%% file: nips-grpo-dynamics.tex
\documentclass{article}
\usepackage[preprint]{neurips_2026}

\usepackage[utf8]{inputenc}
\usepackage[T1]{fontenc}
\usepackage{hyperref}
\usepackage{url}
\usepackage{booktabs}
\usepackage{amsfonts}
\usepackage{amsmath}
\usepackage{amssymb}
\usepackage{nicefrac}
\usepackage{microtype}
\usepackage{xcolor}
\usepackage{graphicx}
\usepackage{algorithm}
\usepackage{algorithmic}
\usepackage{cleveref}
\usepackage{amsthm}

\newtheorem{theorem}{Theorem}
\newtheorem{proposition}[theorem]{Proposition}
\newtheorem{corollary}[theorem]{Corollary}
\newtheorem{definition}[theorem]{Definition}

\input{math_commands.tex}

\title{Gradient Starvation in Binary-Reward GRPO:\\Why Group-Mean Centering Fails and Why the Simplest Fix Works}

\author{%
\textbf{Wenhua Nie \quad Jianan Wu \quad Junlin Liu \quad Ziwei Li \quad Zheng Lin}\\
\textbf{Zhang Zijian \quad Yilong Fan \quad Haoran Zheng \quad Jyh-Shing Roger Jang}\\[3pt]
\normalfont Correspondence: Wenhua Nie, National Taiwan University\\
\normalfont \texttt{d13944014@ntu.edu.tw}
}

\begin{document}

\maketitle

\input{sections/0_abstract.tex}
\input{sections/1_introduction.tex}
\input{sections/2_related_work.tex}
\input{sections/3_method.tex}
\input{sections/4_experiments.tex}
\input{sections/5_analysis.tex}
\input{sections/6_conclusion.tex}

\clearpage
\bibliographystyle{plainnat}
\bibliography{references}

\newpage
\appendix
\input{sections/A_appendix.tex}

\end{document}

%% file: math_commands.tex
\newcommand{\E}{\mathbb{E}}
\newcommand{\Var}{\mathrm{Var}}

\newcommand{\px}{p_x}
\newcommand{\passk}{\mathrm{pass@}k}
\newcommand{\passG}{\mathrm{pass@}G}
\newcommand{\Dreal}{D_{\mathrm{real}}}
\newcommand{\Diid}{D_{\mathrm{iid}}}

%% file: sections/0_abstract.tex
\begin{abstract}
Group Relative Policy Optimization (GRPO) is a standard algorithm for reinforcement learning from verifiable rewards, but its group-mean-centered advantage can fail under binary rewards.
The failure mode is \emph{gradient starvation}: when every response in a group is correct or every response is wrong, the centered advantage is exactly zero and the policy receives no learning signal.
We prove that the true degeneracy rate always exceeds the i.i.d. Bernoulli prediction by Jensen's inequality, and observe a 0.69 degeneracy rate at group size four in logged Qwen3.5-9B GSM8K training.
We then show that the fixed-reference Sign advantage, $A=2r-1$, performs pass@$G$ failure descent by increasing the probability that at least one sample in the group succeeds.
On the full GSM8K test set across seven seeds, Sign reaches 73.8\% accuracy versus 28.4\% for standard normalized group-mean DrGRPO at group size four, a 45.4 point gain with $p<0.0001$.
The effect is directionally consistent on Llama-3.1-8B and positive but underpowered on a MATH-500 transfer check.
Pass@$k$ analysis indicates that the main benefit is search compression rather than large capacity expansion, aligning the empirical gains with recent RLVR ceiling observations.
\end{abstract}

%% file: sections/1_introduction.tex
\section{Introduction}
\label{sec:intro}

GRPO~\citep{shao2024deepseekmath} enables mathematical reasoning post-training by generating multiple responses per prompt, scoring them with a verifiable reward, and computing a group-relative advantage for policy gradient updates.
Under binary verifiable rewards $r \in \{0,1\}$ (correct or incorrect), however, a fundamental failure mode emerges: \emph{gradient starvation}.

When all $G$ responses in a group receive the same binary reward---either all correct or all incorrect---the group-mean advantage $A_j = r_j - \bar{r}$ is exactly zero for every response.
The policy gradient contribution from such \emph{degenerate groups} vanishes entirely.
For a model with per-prompt accuracy $\px$ at small $G$, the probability of a degenerate group is $\px^G + (1{-}\px)^G$, which is large whenever $\px$ is close to 0 or 1.
We show that the observed degeneracy rate reaches $0.69$ at $G{=}4$ for Qwen3.5-9B on GSM8K in 800 logged groups from a seed-42 DrGRPO trajectory---meaning that over two-thirds of logged groups produce zero gradient under standard centering, and learning depends on the remaining minority of mixed groups.
This failure is group-size dependent rather than absolute: when we increase to $G{=}8$, DrGRPO recovers to 81.7\% accuracy over seven seeds, but matched Sign runs still reach 85.8\% over five seeds, indicating that larger groups reduce starvation while fixed-reference advantages can remain substantially more effective.

Prior GRPO variants address related but distinct issues.
Dr.GRPO~\citep{liu2025drgrpo} removes length bias via per-token normalization.
DAPO~\citep{yu2025dapo} uses dynamic sampling to avoid degenerate groups at generation time.
Recent work has also noted that low within-group reward variation can erase GRPO learning signal~\citep{xu2025spo,zhong2026rcgrpo}.
What remains missing is a per-prompt quantification of this effect for binary-reward RLVR, a formal connection to pass@$G$, and a direct comparison of fixed-reference advantage choices.
Large groups and dynamic sampling are valid ways to reduce the number of degenerate groups, but they do not remove the advantage-level mechanism.
They instead provide a useful prediction to test: if gradient starvation is the cause, then increasing $G$ should rescue group-mean centering, while fixed-reference advantages should remain competitive at matched $G$ because they retain signal on the residual degenerate groups.

Our key insight is that the fix for gradient starvation is remarkably simple: replace the group-mean reference with a fixed threshold.
The Sign advantage $A_j = 2r_j - 1$ assigns $+1$ to correct and $-1$ to incorrect responses regardless of group composition, providing full-strength gradient signal in every group including degenerate ones.
We prove that this choice is not merely a heuristic: the expected gradient from all-fail groups under Sign advantage is proportional to $\nabla(1{-}p)^G$, performing gradient ascent on $\passG = 1 - (1{-}p)^G$ (\Cref{thm:passg}).
Our contribution is not the fix itself---which is trivially simple---but the diagnosis of \emph{why} standard GRPO loses signal under binary rewards, the formal connection between degenerate-group signal and pass@$k$ expansion, and the empirical demonstration that coverage of degenerate groups (not expected gradient magnitude) determines learning outcomes.

\paragraph{Contributions.}
\begin{enumerate}
    \item \textbf{Gradient starvation diagnosis.} We provide the first per-prompt quantification of degeneracy in binary-reward GRPO, proving via Jensen's inequality that $\Dreal \geq \Diid$, measuring $\Dreal = 0.90$ at $G{=}2$, and observing a $0.69$ degeneracy rate at $G{=}4$ in 800 logged groups from a seed-42 DrGRPO trajectory (\Cref{sec:starvation}).
    \item \textbf{Pass@$G$ failure descent theorem.} We prove that all-fail groups under fixed-reference advantage produce a gradient proportional to $\nabla(1{-}p)^G$, directly increasing $\passk$ for every $k \geq 1$ (\Cref{thm:passg}).
    \item \textbf{Empirical validation and larger-group control.} We compare six advantage formulations across 7 seeds on GSM8K ($n{=}1319$), finding Sign achieves 73.8\% versus DrGRPO 28.4\% at $G{=}4$ ($+45.4$pp, $p < 0.0001$). A matched larger-group control shows DrGRPO improves to 81.7\% at $G{=}8$ while matched Sign runs reach 85.8\%, confirming the predicted group-size transition and sharpening the claim to small-group starvation and same-$G$ advantage design. The effect is directionally consistent across Llama-3.1-8B (Sign $+11.3$pp over base), with gains largest for moderately weak models above a minimal capability threshold; a single $G{=}4$ MATH-500 transfer check is positive but underpowered (\Cref{sec:experiments}).
\end{enumerate}

\begin{figure}[t]
    \centering
    \includegraphics[width=\textwidth]{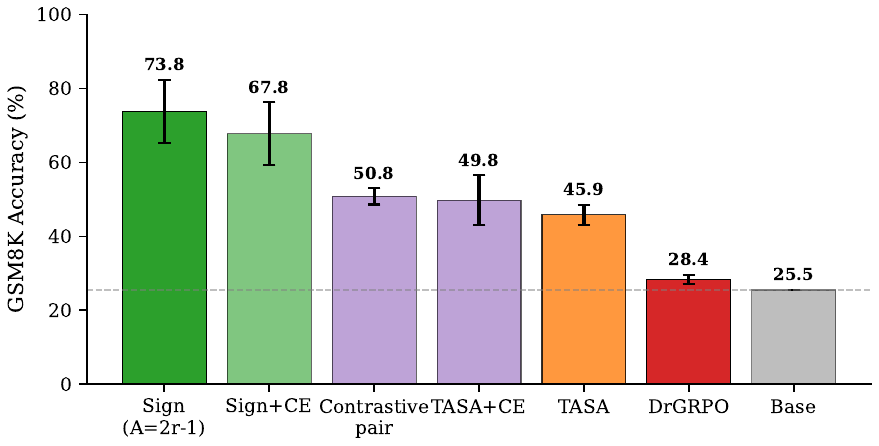}
    \caption{GSM8K accuracy (\%) for six advantage formulations under binary-reward GRPO (Qwen3.5-9B, $G{=}4$, 200 steps, 7 seeds). Std-normalized group-mean centering (DrGRPO) stays near base ($+$2.9pp), while the simplest fixed-reference advantage $A{=}2r{-}1$ (Sign) recovers $+$45.4pp. Error bars: standard deviation.}
    \label{fig:hero}
\end{figure}

%% file: sections/2_related_work.tex
\section{Related Work}
\label{sec:related}

\paragraph{GRPO and variants for RLVR.}
GRPO~\citep{shao2024deepseekmath} introduced group-relative advantage estimation for reinforcement learning from verifiable rewards, eliminating the need for a learned value function by using outcome-based binary verification.
Dr.GRPO~\citep{liu2025drgrpo} addressed length bias by removing per-sequence length normalization and using per-token KL penalties.
DAPO~\citep{yu2025dapo} proposed dynamic sampling to filter degenerate groups at generation time, clip-higher to allow larger policy updates, and overlong reward shaping.
SPO~\citep{xu2025spo} explicitly identifies degenerate groups as wasted compute and replaces per-group baselines with a persistent value tracker; RC-GRPO studies reward-conditioned rollouts in multi-turn tool-calling agents~\citep{zhong2026rcgrpo}.
These variants show that vanishing group-relative signal is an active concern.
Our focus is complementary: we quantify per-prompt binary-reward degeneracy via Jensen's inequality, prove its pass@$G$ connection, and isolate the advantage-level coverage difference between Sign, TASA, and std-normalized DrGRPO.

\paragraph{Advantage normalization in policy gradient methods.}
The choice of advantage baseline profoundly affects policy gradient estimator variance~\citep{schulman2016gae,williams1992reinforce}.
Classical REINFORCE uses a state-dependent baseline; PPO~\citep{schulman2017ppo} uses a learned value function.
GRPO's group-mean centering is a sample-based analog: the baseline is the within-group reward mean.
Under binary rewards, this baseline coincides with the empirical group accuracy, which is 0 or 1 in degenerate groups, producing zero advantage.
The Sign advantage $A = 2r - 1$ is equivalent to using a fixed baseline of $b = 0.5$; to our knowledge, this fixed-reference choice has not been analyzed as a degeneracy-coverage mechanism in binary-reward RLVR.
We note explicitly that Sign is not a novel algorithm---it is the simplest possible fixed-reference advantage.
Our contribution is instead the analysis: quantifying how group-mean centering loses signal under binary rewards (0.69 observed degeneracy at $G{=}4$ in the main DrGRPO trajectory), proving its connection to pass@$G$, and showing that degenerate-group \emph{coverage} rather than expected gradient magnitude determines learning.
DAPO's dynamic sampling addresses degeneracy at generation time by resampling until mixed groups are obtained; our analysis explains \emph{why} this helps (degenerate groups carry zero gradient under centering) and proposes a complementary training-time approach.
Because Sign changes only the advantage reference, it adds no extra generations and is orthogonal to generation-time filtering.
A direct empirical comparison between the two strategies (generation-time filtering vs.\ training-time fixed reference) remains important future work.

\paragraph{RLVR capacity and search compression.}
Recent work by~\citet{bpr2025} (``Does RL Really Incentivize Reasoning Capacity?'') demonstrated that RLVR primarily acts as search compression: pass@1 improves while pass@$k$ remains bounded by the base model's latent capacity.
Our pass@$k$ analysis at $k{=}128$ is consistent with this finding ($\Delta$pass@1 $= +11.9$pp compression, $\Delta$pass@100 $= +2.2$pp expansion), and our pass@$G$ failure descent theorem provides a mechanistic explanation for why RLVR can nonetheless produce small capacity expansion through the degenerate-group gradient channel.

\paragraph{Replay and experience management in RLVR.}
ExGRPO organizes and prioritizes useful past experiences for reasoning-oriented RL training~\citep{exgrpo2025}.
Our ablation shows that CE replay does not improve Sign advantage (67.8\% vs.\ 73.8\%, $p = 0.249$), suggesting that when the advantage signal is already strong, replay adds no benefit and may introduce noise from stale trajectories.

%% file: sections/3_method.tex
\section{Gradient Starvation and Its Cure}
\label{sec:method}

\subsection{Gradient Starvation in Binary-Reward GRPO}
\label{sec:starvation}

Consider GRPO with binary reward $r \in \{0,1\}$, group size $G$, and per-prompt policy success probability $\px = \pi_\theta(C_x \mid x)$ where $C_x$ is the set of correct completions for prompt $x$.
The standard group-mean advantage for response $j$ in a group is:
\begin{equation}
    A_j = r_j - \frac{1}{G}\sum_{i=1}^G r_i = r_j - \frac{n_+}{G},
    \label{eq:drgrpo_advantage}
\end{equation}
where $n_+ = \sum_i r_i$ is the number of correct responses.
Our implemented DrGRPO baseline further divides this centered advantage by the within-group reward standard deviation on mixed groups; the degeneracy argument is unchanged because all-fail and all-pass groups still have zero centered numerator.

\begin{definition}[Degenerate group]
A group is \emph{degenerate} if $n_+ = 0$ (all incorrect) or $n_+ = G$ (all correct). In a degenerate group, $A_j = 0$ for all $j$, and the policy gradient contribution is exactly zero.
\end{definition}

For a prompt with success probability $\px$, the degeneracy probability is:
\begin{equation}
    D(\px, G) = \px^G + (1 - \px)^G.
    \label{eq:degeneracy}
\end{equation}

The \emph{real} degeneracy rate across the training set is $\Dreal = \E_x[D(\px, G)]$.

\begin{proposition}[Jensen bound on gradient starvation]
\label{prop:jensen}
Assume the $G$ responses per prompt are sampled i.i.d.\ from $\pi_\theta(\cdot \mid x)$, so that each binary reward $r_i \sim \mathrm{Bernoulli}(\px)$ independently.
Since $D(p, G) = p^G + (1{-}p)^G$ is convex for $G \geq 2$ (its second derivative $G(G{-}1)[p^{G-2} + (1{-}p)^{G-2}] \geq 0$), Jensen's inequality gives:
\begin{equation}
    \Dreal = \E_x[D(\px, G)] \geq D(\bar{p}, G) = \bar{p}^G + (1 - \bar{p})^G = \Diid,
    \label{eq:jensen}
\end{equation}
where $\bar{p} = \E_x[\px]$ is the mean success probability. Furthermore, by second-order Taylor expansion:
\begin{equation}
    \Dreal \geq \Diid + \frac{G(G{-}1)}{2}\Var(\px) \cdot \min_{p \in [0,1]}[p^{G-2} + (1{-}p)^{G-2}].
    \label{eq:variance_bound}
\end{equation}
\end{proposition}

\textbf{Empirical measurement.} For Qwen3.5-9B on GSM8K at $G{=}2$, we measure per-prompt $\px$ from training rollouts and find $\Dreal = 0.90$, while $\Diid = 0.56$ (using $\bar{p} = 0.325$). The Jensen gap is $+60\%$, driven by the highly bimodal distribution of $\px$: 57.5\% of prompts have $\px = 0$ (model never generates correct answers) and 22.5\% have $\px = 1$ (always correct).
At the experimental setting of $G{=}4$, the homogeneous comparator based on the base greedy accuracy rounded to $\bar{p}\approx0.25$ is $\Diid(0.25, 4) = 0.32$, but observed training degeneracy is substantially higher due to the same per-prompt heterogeneity.
In a seed-42 DrGRPO run, the trainer logs 800 groups across 200 optimizer steps: $69.25\%$ are degenerate ($54.75\%$ all-fail, $14.50\%$ all-pass), with early-training degeneracy often exceeding $75\%$ before the model improves (\Cref{sec:experiments}).
We use this as an empirical trajectory-level estimate of the training regime rather than a seed-averaged estimate of $\Dreal$.
Under group-mean centering, these degenerate groups produce \emph{zero gradient}, and learning depends entirely on the minority of mixed groups.

\subsection{Fixed-Reference Advantage and Pass@$G$ Failure Descent}
\label{sec:passg}

The Sign advantage replaces the group-mean reference with a fixed threshold:
\begin{equation}
    A_j^{\mathrm{Sign}} = 2r_j - 1 = \begin{cases} +1 & \text{if } r_j = 1 \text{ (correct)} \\ -1 & \text{if } r_j = 0 \text{ (incorrect)} \end{cases}.
    \label{eq:sign}
\end{equation}

This provides non-zero gradient in \emph{every} group, including degenerate ones: all-correct groups receive $A = +1$ (reinforcement) and all-incorrect groups receive $A = -1$ (avoidance).

\begin{theorem}[Pass@$G$ failure descent]
\label{thm:passg}
Let $p = \pi_\theta(C_x) \in (0,1)$, $q = 1 - p$, and $s_\theta(y) = \nabla_\theta \log \pi_\theta(y \mid x)$ be the score function.
Consider a fixed-reference advantage with $c > 0$ where incorrect responses receive $A = -c$ and correct responses receive $A = +c$.
Assume the $G$ responses are sampled i.i.d.\ from $\pi_\theta(\cdot \mid x)$ and that $\E_{y \sim \pi_\theta}[\|s_\theta(y)\|]$ is finite (regularity).
The on-policy policy gradient loss is $L = -\frac{1}{G}\sum_i A_i \log \pi_\theta(Y_i \mid x)$.
In an all-fail group ($N = 0$), every $A_i = -c$, so $L_0 = (c/G)\sum_i \log \pi_\theta(Y_i \mid x)$.
The expected gradient from all-fail groups satisfies:
\begin{equation}
    \E\left[\nabla_\theta L_0 \cdot \mathbf{1}\{N = 0\}\right] = -c \cdot q^{G-1} \nabla_\theta p = \frac{c}{G} \nabla_\theta q^G,
    \label{eq:failure_descent}
\end{equation}
where $N = \sum_{i=1}^G r_i$. Gradient descent on this term performs gradient ascent on $\passG = 1 - q^G$, locally increasing the probability that at least one of $G$ samples succeeds.
\end{theorem}

\begin{proof}[Proof sketch]
Conditioned on $N = 0$, each sample $Y_i \sim \pi_\theta(\cdot \mid Y \in W_x)$ where $W_x = C_x^c$.
The expected score is $\E[s_\theta(Y) \mid Y \in W_x] = \nabla_\theta \log q$.
With $A_i = -c$ for all $i$ and $\Pr(N = 0) = q^G$:
$\E[\nabla_\theta L_0 \cdot \mathbf{1}\{N{=}0\}] = c \cdot q^G \cdot \nabla_\theta \log q = c \cdot q^{G-1} \cdot (-\nabla_\theta p)$.
Since $\nabla_\theta q^G = G q^{G-1} \nabla_\theta q = -G q^{G-1} \nabla_\theta p$, the result follows.
The full proof is provided in the appendix.
\end{proof}

For Sign advantage ($c = 1$), the all-fail gradient magnitude is $q^{G-1}|\nabla_\theta p|$---full strength.

\begin{corollary}
\label{cor:passk}
If $0 < p < 1$ and $\nabla_\theta p \neq 0$, the all-fail Sign gradient locally increases $\passk$ for every $k \geq 1$.
\end{corollary}

\subsection{Why Coverage, Not Magnitude, Determines Performance}
\label{sec:coverage}

The TASA baseline used below is the threshold-anchored signed advantage with threshold $c=1/2$:
\begin{equation}
    A_i^{\mathrm{TASA}} =
    \frac{(r_i-c)_+}{\sum_k (r_k-c)_+}\mathbf{1}\left\{\sum_k (r_k-c)_+>0\right\}
    -
    \frac{(c-r_i)_+}{\sum_k (c-r_k)_+}\mathbf{1}\left\{\sum_k (c-r_k)_+>0\right\}.
    \label{eq:tasa}
\end{equation}
For binary rewards, mixed groups assign $+1/n_+$ to correct responses and $-1/(G-n_+)$ to incorrect responses; all-fail groups assign $-1/G$ to every response and all-pass groups assign $+1/G$.

The contrastive-pair replay baseline keeps the same on-policy signed advantage but adds a prompt-local logistic replay loss over archived verified positive-negative completion pairs (formalized in \Cref{app:details}); it tests whether replayed pair evidence can replace or augment degenerate-group gradient coverage.

A natural question is whether the ranking Sign $\gg$ TASA $\gg$ DrGRPO is explained by expected gradient magnitude.
It is not.
At $p \approx 0.25$ and $G = 4$, the expected policy gradient coefficients are: Sign $= 2.0|\nabla p|$, std-normalized DrGRPO $= 1.425|\nabla p|$, TASA $= 1.016|\nabla p|$.
DrGRPO has \emph{higher} expected gradient than TASA, yet TASA empirically outperforms DrGRPO (45.9\% vs.\ 28.4\%).

The resolution is that expected magnitude averages over all group compositions, while \emph{degenerate-group coverage} determines actual learning.
With the observed $0.69$ degeneracy rate at $G{=}4$ (from the seed-42 DrGRPO training log), degenerate groups dominate:
\begin{itemize}
    \item \textbf{Sign}: all-fail groups produce gradient $\pm 1$ per response---full coverage.
    \item \textbf{TASA}: all-fail groups produce gradient $\pm 1/G$---attenuated but non-zero.
    \item \textbf{DrGRPO}: all-fail groups produce gradient $0$---zero coverage.
\end{itemize}
When the majority of groups are degenerate, the non-degenerate expected magnitude is irrelevant; what matters is whether the method extracts \emph{any} signal from degenerate groups.

\paragraph{Quantifying the degenerate-group contribution.}
From \Cref{thm:passg}, the expected gradient contribution from all-fail groups under a method with degenerate-group advantage magnitude $a$ is $-a \cdot q^{G-1} \nabla_\theta p$; similarly, all-pass groups contribute $-a \cdot p^{G-1} \nabla_\theta p$.
The total degenerate-group gradient is thus proportional to $(q^{G-1} + p^{G-1}) \|\nabla_\theta p\|$.
For DrGRPO, $a = 0$ and this contribution vanishes entirely.
For Sign, $a = 1$; for TASA, $a = 1/G$.
Since degenerate groups constitute $69\%$ of logged groups in that $G{=}4$ trajectory and the model spends most of training with $p$ near 0 (where $q^{G-1} \approx 1$), any method with $a = 0$ loses the majority of available gradient signal regardless of its behavior in non-degenerate groups.

%% file: sections/4_experiments.tex
\section{Experiments}
\label{sec:experiments}

\subsection{Setup}

We train Qwen3.5-9B~\citep{qwen35modelcard,qwen2025qwen3} with LoRA~\citep{hu2021lora} ($r{=}64$, $\alpha{=}128$) on the GSM8K~\citep{cobbe2021gsm8k} training set (7,473 problems) using binary reward verification.
All methods use $G{=}4$, learning rate $2 \times 10^{-5}$, 200 training steps, and the same TRL-based trainer implementation to ensure identical data pipeline, generation, and evaluation.
The primary variable is the advantage formulation; replay variants additionally include a CE replay loss ($\lambda_{\mathrm{CE}}{=}0.05$).
We evaluate on the full GSM8K test set ($n{=}1{,}319$) with greedy decoding and report results across 7 random seeds (42--48).

\subsection{Main Results}

\Cref{tab:main} presents accuracy for all six advantage formulations.
Sign advantage achieves 73.8\%, a $+$45.4pp improvement over DrGRPO ($p < 0.0001$, Welch's $t$-test) and $+$48.3pp over the base model.
DrGRPO barely learns ($+$2.9pp over base), consistent with severe gradient starvation in this small-$G$, finite-budget setting.
Although Sign has high seed variance, the separation is not driven solely by an outlier: its median is 70.7\%, and its worst seed (64.6\%) still exceeds the best DrGRPO seed (30.6\%).

\begin{table}[t]
\centering
\caption{GSM8K test accuracy (\%, $n{=}1{,}319$) for different advantage formulations. All methods use Qwen3.5-9B, $G{=}4$, $\mathrm{lr}{=}2{\times}10^{-5}$, 200 steps. Mean, median, and population std over 7 seeds (sample std is $\approx$8\% larger at $n{=}7$). $^\ast$: Welch's $t$-test vs.\ DrGRPO.}
\label{tab:main}
\begin{tabular}{lccccl}
\toprule
Method & $N$ & Mean (Med.) & Std & Range & $p$-value$^\ast$ \\
\midrule
Sign ($A{=}2r{-}1$) & 7 & \textbf{73.8} (70.7) & 8.6 & 64.6--90.1 & $<$0.0001 \\
Sign + CE replay & 7 & 67.8 (68.5) & 8.5 & 54.8--80.9 & $<$0.0001 \\
Contrastive pair & 7 & 50.8 (51.2) & 2.2 & 47.1--54.4 & $<$0.0001 \\
TASA + CE replay & 7 & 49.8 (48.4) & 6.8 & 42.2--62.5 & $<$0.001 \\
TASA (threshold anchored) & 7 & 45.9 (45.8) & 2.7 & 41.2--50.9 & $<$0.0001 \\
DrGRPO (std-norm group mean) & 7 & 28.4 (28.4) & 1.2 & 26.5--30.6 & baseline \\
\midrule
Base (no training) & -- & 25.5 & -- & -- & -- \\
\bottomrule
\end{tabular}
\end{table}

The ranking Sign $\gg$ TASA $\gg$ DrGRPO directly matches the degenerate-group coverage prediction from \Cref{sec:coverage}: Sign provides $\pm 1$ signal in all groups, TASA provides $\pm 1/G$ in degenerate groups, contrastive-pair replay provides only archived pair signal, and DrGRPO provides zero.

\subsection{Matched Larger-Group Control}
\label{sec:g8control}

The $G{=}4$ comparison raises an important alternative explanation: perhaps group-mean centering only fails because the group is too small.
We therefore ran a larger-group control at $G{=}8$ with the same model, dataset, step count, group size, and greedy GSM8K evaluation protocol.
\Cref{tab:g8control} shows that increasing $G$ substantially mitigates DrGRPO's failure: DrGRPO rises from 28.4\% at $G{=}4$ to 81.7\% at $G{=}8$ over seven seeds.
This result is positive for the gradient-starvation diagnosis because the predicted degeneracy probability $p^G + (1-p)^G$ decreases as $G$ grows for intermediate prompt success probabilities.
It also rules out an overstrong claim that group-mean centering is intrinsically unusable under binary rewards.
The practical implication is a tradeoff rather than a contradiction: one can spend more generation compute to reduce starvation, use generation-time filtering such as dynamic sampling, or change the advantage reference so that degenerate groups are no longer discarded.
Our experiments isolate the third option.

At the same time, matched Sign runs at $G{=}8$ reach 85.8\% over five seeds, $+4.1$pp over the DrGRPO $G{=}8$ mean; every Sign $G{=}8$ seed exceeds every DrGRPO $G{=}8$ seed.
The median comparison is 84.2\% versus 81.8\%, so the conclusion does not rely on the high seed-42 Sign run.
The separation remains positive under an exact label-permutation test over the 12 completed $G{=}8$ runs ($p=0.0013$).
As a same-group-size stress test of the mechanism, this result shows that larger $G$ rescues much of DrGRPO's lost signal, but fixed-reference advantage can still provide a large additional gain by retaining signal in all-fail and all-pass groups.

\begin{table}[t]
\centering
\caption{Larger-group control on GSM8K (Qwen3.5-9B, 200 steps, full test set). Mean $\pm$ population std with medians in parentheses; deltas are computed before rounding. Appendix details describe the $G{=}8$ training configuration.}
\label{tab:g8control}
\begin{tabular}{lcccc}
\toprule
Method & $G$ & $N$ & Accuracy (Med.) & Range \\
\midrule
DrGRPO (std-norm group mean) & 4 & 7 & 28.4 $\pm$ 1.2 (28.4) & 26.5--30.6 \\
Sign ($A{=}2r{-}1$) & 4 & 7 & 73.8 $\pm$ 8.6 (70.7) & 64.6--90.1 \\
DrGRPO (std-norm group mean) & 8 & 7 & 81.7 $\pm$ 0.4 (81.8) & 80.9--82.2 \\
Sign ($A{=}2r{-}1$) & 8 & 5 & 85.8 $\pm$ 4.0 (84.2) & 82.6--93.6 \\
\bottomrule
\end{tabular}
\end{table}

\subsection{Pass@$k$ Analysis}

Appendix~\ref{app:passk} reports the full pass@$k$ table and curves. The key pattern is that Sign improves pass@1 by $+$11.9pp but pass@100 by only $+$2.2pp, consistent with search compression rather than large capacity expansion~\citep{bpr2025}. TASA and Sign reach the same pass@100 despite a large greedy-accuracy gap, reinforcing that Sign mainly makes existing solution mass easier to find.

\subsection{Gradient Starvation Visualization}

\Cref{fig:starvation} visualizes the gradient starvation phenomenon at $G{=}2$.
The per-prompt success rate $\px$ at training start is highly bimodal: 57.5\% of prompts have $\px = 0$ and 22.5\% have $\px = 1$.
The measured $\Dreal = 0.90$ at $G{=}2$ exceeds $\Diid = 0.56$ by 60\%, confirming the Jensen bound (\Cref{prop:jensen}).
At the experimental setting of $G{=}4$, per-prompt heterogeneity similarly amplifies degeneracy: in the seed-42 DrGRPO trajectory, 800 logged groups across 200 optimizer steps have observed degeneracy $0.6925$ ($54.75\%$ all-fail, $14.50\%$ all-pass), versus $\Diid(0.25, 4) = 0.32$---a $2.2\times$ amplification over the homogeneous Bernoulli prediction.

\begin{figure}[t]
    \centering
    \includegraphics[width=\textwidth]{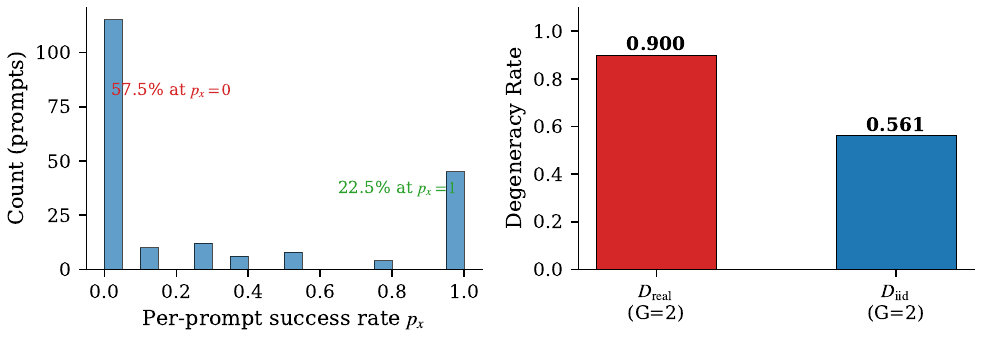}
    \caption{\textbf{Left:} Per-prompt success-rate distribution at training start ($G{=}2$), with endpoint masses measured from rollouts and intermediate bins shown for visualization. 57.5\% of prompts have $\px{=}0$. \textbf{Right:} Measured degeneracy $\Dreal{=}0.90$ vs.\ i.i.d.\ prediction $\Diid{=}0.56$, a 60\% Jensen gap caused by the bimodal $\px$ distribution.}
    \label{fig:starvation}
\end{figure}

\subsection{Exploratory Cross-Model and Cross-Family Checks}

To probe generality, we evaluate Sign on three additional models spanning two families and a wide capability range (\Cref{tab:crossmodel}).
The pattern is consistent: Sign outperforms the base model on every model where the base has nonzero success probability, and among those models the gain decreases with base strength---as predicted by gradient starvation theory.
These checks are directional (2--3 seeds for the non-main models, no DrGRPO control for Llama/Qwen3) rather than statistically validated cross-family comparisons.

On Llama-3.1-8B-Instruct (a non-Qwen model with 70.4\% base accuracy), Sign achieves 81.7\% ($+$11.3pp, 2 seeds), showing the effect is consistent across model families.
On Qwen3-8B (57.5\% base), Sign reaches 75.1\% ($+$17.6pp, 3 seeds).
On the weakest model, Qwen3.5-2B (17.4\% base), Sign shows no improvement---200 training steps are insufficient when the model almost never generates correct answers, as the all-fail gradient requires $p_x > 0$ to have a non-trivial effect.

\begin{table}[t]
\centering
\caption{Exploratory cross-model and cross-family checks (GSM8K, $G{=}4$, 200 steps). Sign advantage is directionally consistent across model families and sizes. Gains are largest for moderate-capability models; very weak models (Qwen3.5-2B) cannot benefit from 200 steps alone.}
\label{tab:crossmodel}
\begin{tabular}{llccc}
\toprule
Model & Family & Base & Sign & $\Delta$ \\
\midrule
Qwen3.5-2B & Qwen & 17.4 & 17.4 & $+$0.0 \\
Qwen3.5-9B & Qwen & 25.5 & \textbf{73.8} ($\pm$8.6) & \textbf{$+$48.3} \\
Qwen3-8B & Qwen & 57.5 & \textbf{75.1} ($\pm$4.7) & $+$17.6 \\
Llama-3.1-8B-Inst & Llama & 70.4 & \textbf{81.7} ($\pm$0.4) & $+$11.3 \\
\bottomrule
\end{tabular}
\end{table}

Above a minimal capability threshold ($\px > 0$ for enough prompts), Sign gain decreases with base strength: Qwen3.5-9B ($+$48.3pp) $>$ Qwen3-8B ($+$17.6pp) $>$ Llama ($+$11.3pp).
Below this threshold, Qwen3.5-2B (base 17.4\%, with $\px \approx 0$ for most prompts) shows no improvement, confirming that Sign's all-fail gradient requires nonzero base success to produce learning signal.

\subsection{Training Dynamics and Transfer}

Training logs match the starvation diagnosis. In the seed-42 trajectory, Sign's reward rises from 0.21 to 0.76 and the all-fail fraction falls from 0.70 to 0.10; DrGRPO stays near 0.25--0.40 reward with degeneracy around 0.69--0.72, showing no escape within 200 steps. A harder single-seed MATH-500 transfer check is directionally positive for $G{=}4$ (Sign 36.8\%, DrGRPO 31.8\%), but we treat it as exploratory because the matched $G{=}8$ MATH-500 result in the appendix is not positive.

%% file: sections/5_analysis.tex
\section{Analysis}
\label{sec:analysis}

\paragraph{Why DrGRPO fails at small group size.}
DrGRPO's weak $G{=}4$ improvement under the 200-step budget ($28.4\%$ vs.\ base $25.5\%$) is explained by gradient starvation.
With the observed $0.69$ degeneracy rate at $G{=}4$ in the seed-42 DrGRPO log, only $31\%$ of logged groups produce non-zero gradient.
The remaining mixed groups are biased toward intermediate-$\px$ prompts; hard ($\px\approx0$) and easy ($\px\approx1$) prompts receive zero centered signal.
The $G{=}8$ control refines rather than contradicts this diagnosis: DrGRPO recovers to 81.7\%, showing a group-size phase transition rather than an impossibility theorem for group-mean centering.
Fixed-reference advantages are therefore most valuable in small- and moderate-$G$ regimes where degenerate groups still consume a large fraction of training compute.
This also explains why dynamic-sampling methods are plausible: they attack the same degeneracy term at generation time, whereas Sign attacks it at the advantage level.

\paragraph{CE replay does not help Sign.}
Sign + CE replay (67.8\%) trends below Sign alone (73.8\%), though the difference is not statistically significant ($p = 0.249$, Welch's $t$-test).
Pass@$k$ analysis supports this interpretation: Sign+CE and Sign both reach pass@100 $=99.4\%$, suggesting replay adds no capacity expansion when the on-policy advantage is already full strength.
In contrast, CE replay improves TASA (49.8\% vs.\ 45.9\%), where the weaker advantage signal benefits from additional supervised signal.

\paragraph{Group size effect.}
Sign at $G{=}2$ achieves 67.1\% (3 seeds) versus 73.8\% at $G{=}4$ (7 seeds), showing that fixed-reference advantage remains useful even when degenerate groups are frequent.
Moving from $G{=}4$ to $G{=}8$ has a much larger effect on DrGRPO: its accuracy increases from 28.4\% to 81.7\%, consistent with the theoretical degeneracy term $p^G + (1-p)^G$ shrinking as $G$ grows for intermediate $\px$.
This larger-$G$ recovery strengthens the mechanistic story by showing the expected direction of the sparse group-size sweep.
Sign $G{=}8$ reaches 85.8\% over five seeds and every Sign $G{=}8$ seed exceeds every DrGRPO $G{=}8$ seed; the median gap (84.2\% vs.\ 81.8\%) shows that this conclusion is not an artifact of the seed-42 outlier.

\paragraph{Limitations.}
Our primary 7-seed evidence is GSM8K with Qwen3.5-9B; cross-model and cross-dataset checks use fewer seeds, and the $G{=}4$ degeneracy measurement is from a single seed-42 trajectory.
A matched $G{=}8$ MATH-500 transfer check is not positive for Sign (29.3\% vs.\ 32.3\% DrGRPO over 3 seeds; see appendix), so cross-dataset generality remains open.
The $G{=}8$ control uses a separate larger-group configuration with gradient accumulation and completion-length settings disclosed in the appendix; it is a same-$G$ mechanism test rather than a completion-length-normalized comparison against the $G{=}4$ table.
Sign has high seed variance on Qwen3.5-9B; one possible explanation is that the stronger per-token signal makes LoRA updates more sensitive to initialization, while the stronger Llama base has much lower Sign variance (std $=0.4$).
The weakest model does not improve after 200 steps, and we cannot rule out DrGRPO at $G{=}4$ closing part of the gap with substantially more training.
The theory assumes on-policy, unclipped gradients without KL penalty, while practical training includes clipping and KL regularization.
We do not compare against DAPO, SPO, or RC-GRPO-style alternatives; matching fixed-reference advantage against generation-time filtering and persistent baselines remains important future work.

%% file: sections/6_conclusion.tex
\section{Conclusion}
\label{sec:conclusion}

We identified gradient starvation as a small-group failure mode of binary-reward GRPO: at $G{=}4$, a seed-42 DrGRPO trajectory shows that $69\%$ of logged groups produce zero gradient under group-mean advantage centering, and DrGRPO learns only $+$2.9pp above the base model under the 200-step budget.
The fix is to remove group-relative centering entirely: the trivially simple Sign advantage $A = 2r - 1$ provides full-strength gradient signal in every group and achieves $+$45.4pp over DrGRPO.
The larger-group control sharpens this conclusion rather than weakening it: increasing to $G{=}8$ lets DrGRPO recover to 81.7\% over seven seeds, while matched Sign runs reach 85.8\% over five seeds, consistent with starvation decreasing as $G$ grows but not disappearing as an advantage-design issue.
We proved that Sign's all-fail gradient performs pass@$G$ failure descent, and pass@$k$ analysis shows that the gain is primarily search compression with small capacity expansion.
Cross-model and capability-sweep checks are directionally consistent, with gains largest for moderately weak models above a minimal capability threshold; the MATH-500 evidence remains a single $G{=}4$ transfer check.

In RLVR with binary verifiable rewards, advantage normalization is therefore a critical design choice: small-$G$ training should avoid group-mean centering, increase group size, or use generation-time filtering.
More efficient RLVR can reduce compute waste and improve useful reasoning systems, but it can also strengthen models used in harmful automated reasoning; we view gradient-starvation diagnostics as transparency tools rather than release recommendations.

%% file: sections/A_appendix.tex
\section{Proof of the Pass@$G$ Failure Descent Theorem}
\label{app:proofs}

\begin{proof}
Fix a prompt $x$ with policy success probability $p = \pi_\theta(C_x \mid x)$ and failure probability $q = 1 - p$.
Let $s_\theta(y) = \nabla_\theta \log \pi_\theta(y \mid x)$ denote the score function.

Consider a group of $G$ i.i.d.\ samples $Y_1, \ldots, Y_G \sim \pi_\theta(\cdot \mid x)$, and assume the expected score norm is finite: $\E_{Y\sim\pi_\theta}[\|\nabla_\theta\log\pi_\theta(Y\mid x)\|] < \infty$.
Let $N = \sum_{i=1}^G r(Y_i)$ be the number of correct responses.
Under fixed-reference advantage with constant $c>0$ (i.e., $A_i = -c$ for incorrect responses), the policy gradient loss from the group is:
\[
L_0 = \frac{c}{G} \sum_{i=1}^G \log \pi_\theta(Y_i \mid x).
\]

Conditioned on $N = 0$ (all-fail group), each $Y_i$ is drawn from $\pi_\theta(\cdot \mid Y \in W_x)$ where $W_x = C_x^c$ is the incorrect set.

The conditional expected score is:
\[
\E[s_\theta(Y) \mid Y \in W_x] = \sum_{y \in W_x} \frac{\pi_\theta(y)}{q} \nabla_\theta \log \pi_\theta(y) = \frac{1}{q} \sum_{y \in W_x} \nabla_\theta \pi_\theta(y) = \frac{\nabla_\theta q}{q} = \nabla_\theta \log q.
\]

Therefore:
\[
\E[\nabla_\theta L_0 \mid N = 0] = c \cdot \nabla_\theta \log q = c \cdot \frac{-\nabla_\theta p}{q}.
\]

Weighting by $\Pr(N = 0) = q^G$:
\begin{align}
\E[\nabla_\theta L_0 \cdot \mathbf{1}\{N = 0\}] &= q^G \cdot c \cdot \frac{-\nabla_\theta p}{q} = -c \cdot q^{G-1} \nabla_\theta p.
\end{align}

Since $\nabla_\theta q^G = G q^{G-1} \nabla_\theta q = -G q^{G-1} \nabla_\theta p$, we have:
\[
\E[\nabla_\theta L_0 \cdot \mathbf{1}\{N = 0\}] = \frac{c}{G} \nabla_\theta q^G = -\frac{c}{G} \nabla_\theta \passG.
\]

Gradient descent on $L_0$ therefore performs gradient ascent on $\passG = 1 - q^G$.
For Sign advantage, $c = 1$, giving full-strength descent on the failure probability $q^G$.

Since $\frac{d}{dp}\passk = k(1-p)^{k-1} > 0$ for $0 < p < 1$, increasing $p$ increases $\passk$ for every $k \geq 1$, which proves the pass@$k$ monotonicity corollary.
\end{proof}

\section{Additional Experimental Details}
\label{app:details}

\subsection{Pass@$k$ Details}
\label{app:passk}

\begin{table}[h]
\centering
\caption{Pass@$k$ comparison ($k{=}128$, 200 random GSM8K questions, temperature 1.0, one checkpoint per method). Deltas are computed before rounding displayed values.}
\label{tab:passk}
\begin{tabular}{lccccc}
\toprule
& pass@1 & pass@10 & pass@25 & pass@50 & pass@100 \\
\midrule
Base & 81.4 & 93.9 & 95.7 & 96.5 & 97.2 \\
DrGRPO & 85.3 & 95.9 & 97.6 & 98.5 & 99.0 \\
TASA & 88.6 & 97.9 & 98.7 & 99.0 & 99.4 \\
Sign & \textbf{93.2} & \textbf{98.2} & \textbf{98.7} & \textbf{99.0} & \textbf{99.4} \\
\midrule
$\Delta$ (Sign$-$Base) & $+$11.9 & $+$4.3 & $+$3.0 & $+$2.5 & $+$2.2 \\
\bottomrule
\end{tabular}
\end{table}

\begin{figure}[h]
\centering
\includegraphics[width=.48\textwidth]{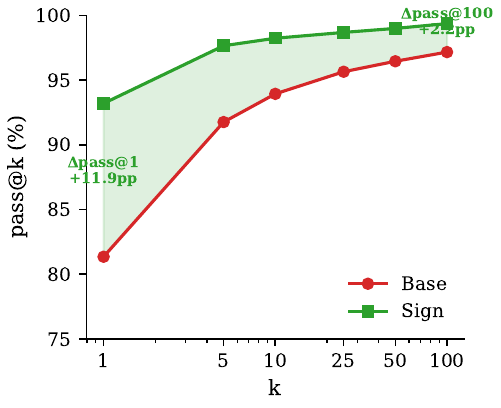}\hfill
\includegraphics[width=.48\textwidth]{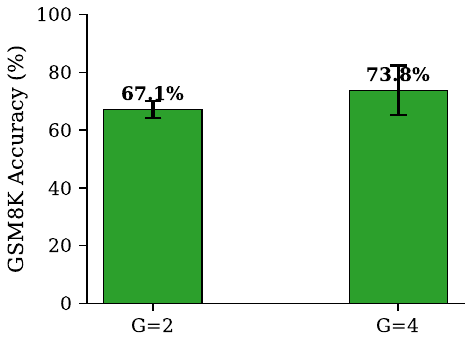}
\caption{Left: Pass@$k$ curves. Right: Sign advantage at $G{=}2$ and $G{=}4$.}
\label{fig:passk}
\end{figure}

\subsection{Training-Dynamics Figure}

\begin{figure}[h]
\centering
\includegraphics[width=\textwidth]{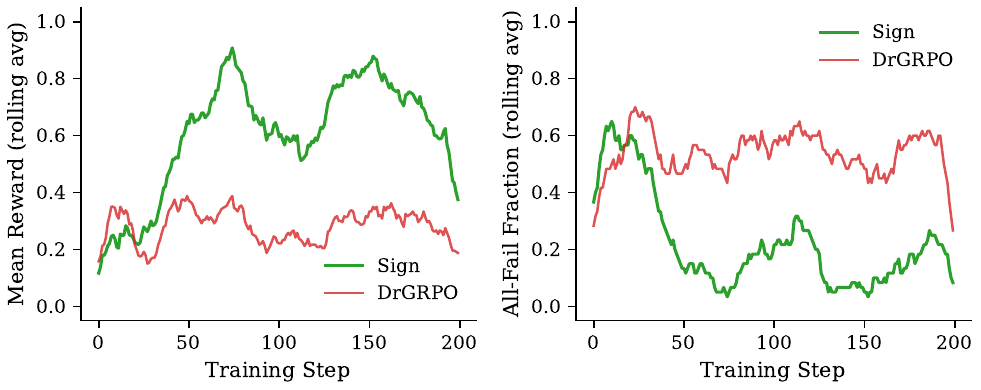}
\caption{\textbf{Left:} Training reward over 200 steps. \textbf{Right:} Fraction of all-fail groups. Sign rapidly escapes gradient starvation, while DrGRPO remains trapped with high all-fail mass.}
\label{fig:training}
\end{figure}

\paragraph{Training configuration.}
The main $G{=}4$ experiments use our TRL-based GRPO trainer~\citep{vonwerra2022trl} with the following shared configuration: LoRA rank $r{=}64$, $\alpha{=}128$, target modules \{q, k, v, o, gate, up, down\}\_proj, dropout 0.05, batch size 1, gradient accumulation 4, warmup ratio 0.05, weight decay 0.01, max gradient norm 1.0, bf16 precision, max completion length 256 tokens.
The only variable across methods is the advantage computation.

\paragraph{Contrastive-pair replay.}
For the ``Contrastive pair'' row, the on-policy advantage is the same threshold-anchored signed advantage as TASA, plus a prompt-local replay loss over archived verified completions.
For a same-prompt positive-negative pair $(y^+,y^-)$, let
\[
m_\theta = [\log\pi_\theta(y^+\mid x)-\log\pi_\theta(y^-\mid x)]-\rho_{\mathrm{ref}}[\log\pi_{\mathrm{ref}}(y^+\mid x)-\log\pi_{\mathrm{ref}}(y^-\mid x)].
\]
The replay term is $\mathcal{L}_{\mathrm{pair}}=\lambda_{\mathrm{pair}}\sum_{(y^+,y^-)} w(y^+,y^-)(-\log\sigma(m_\theta)) / \sum w$, where $w=\mathrm{clip}(\Delta r \cdot f_{\mathrm{frontier}}\cdot d_{\mathrm{age}},0.05,1.0)$.
Here $\Delta r$ is the reward gap, $f_{\mathrm{frontier}}=4\hat{p}_x(1-\hat{p}_x)\min(1,n_x/5)$ uses the prompt Beta-posterior mean $\hat{p}_x$ and observation count $n_x$, and $d_{\mathrm{age}}=\exp[-(a_+ + a_-)/(2\tau)]$ with $\tau=200$ decays stale pairs by their positive/negative ages.

\paragraph{Expected gradient coefficient calculation.}
The coefficients in \Cref{sec:coverage} are computed for the averaged group loss by summing over $N=n$ correct samples.
For Sign, $\E[(1/G)\sum_i A_i s_i]=2\nabla p$.
For TASA at $G{=}4,p{=}0.25,q{=}0.75$, degenerate groups contribute $(q^3+p^3)/4$ and mixed groups contribute $(1-p^4-q^4)(1/p+1/q)/4$, giving $1.016\nabla p$.
For std-normalized DrGRPO, mixed groups use $A^+=(1-n/G)/\hat{\sigma}_n$ and $A^-=-n/(G\hat{\sigma}_n)$ with the unbiased PyTorch group std $\hat{\sigma}_n^2=[n(1-n/G)^2+(G-n)(n/G)^2]/(G-1)$; summing $\binom{4}{n}p^n q^{4-n}(1/4)[nA^+/p+(4-n)(-A^-)/q]$ over $n=1,2,3$ gives $1.425\nabla p$.

\paragraph{Evaluation.}
Greedy decoding on the full GSM8K test set ($n{=}1{,}319$).
Answer extraction uses the \texttt{\#\#\#\#} delimiter with numeric canonicalization (removing trailing zeros and dots).
Pass@$k$ evaluation uses temperature 1.0, top-$p$ 0.95, max 512 new tokens, with the unbiased estimator from~\citet{chen2021codex}.
The MATH-500 transfer check uses the same checkpoint-selection policy but a separate MATH evaluation pipeline, so we report it as a single-seed directional transfer result.
The $G{=}2$ endpoint masses used for the starvation figure are provenance statistics computed from the phase-sweep training-start rollout logs for \texttt{B\_tasa\_g2\_seed42}; they are used only for the Jensen-gap diagnostic, while the main GSM8K claims use archived full-test evaluation JSONs.

\paragraph{Compute.}
All experiments ran on 8$\times$ A800-SXM4-80GB GPUs.
Each 200-step training run takes approximately 3.5 GPU-hours at $G{=}4$.
Pass@$k$ evaluation at $k{=}128$ over 200 questions takes approximately 24 GPU-hours per condition.

\section{Larger-Group Control Details}
\label{app:g8}

\Cref{tab:g8perseed} reports the per-run values for the matched $G{=}8$ control summarized in the main paper.
The DrGRPO $G{=}8$ row is stable across seven seeds, while the Sign $G{=}8$ row has higher variance across five seeds.
The control is matched on model, dataset, group size, step count, and evaluation protocol.
The $G{=}8$ training configuration uses gradient accumulation 8 and a 192-token completion cap to keep the larger group size within memory.
It is therefore a same-$G$ control rather than a completion-length-normalized comparison against the $G{=}4$ main table.

\begin{table}[h]
\centering
\caption{Per-run GSM8K greedy accuracy for the matched $G{=}8$ larger-group control.}
\label{tab:g8perseed}
\begin{tabular}{lcc}
\toprule
Run & Seed & Accuracy (\%) \\
\midrule
DrGRPO $G{=}8$ & 42 & 82.18 \\
DrGRPO $G{=}8$ & 43 & 81.88 \\
DrGRPO $G{=}8$ & 44 & 80.89 \\
DrGRPO $G{=}8$ & 45 & 81.73 \\
DrGRPO $G{=}8$ & 46 & 82.11 \\
DrGRPO $G{=}8$ & 47 & 81.80 \\
DrGRPO $G{=}8$ & 48 & 81.35 \\
Sign $G{=}8$ & 42 & 93.63 \\
Sign $G{=}8$ & 43 & 84.08 \\
Sign $G{=}8$ & 44 & 84.61 \\
Sign $G{=}8$ & 45 & 82.64 \\
Sign $G{=}8$ & 46 & 84.15 \\
\bottomrule
\end{tabular}
\end{table}

\section{Matched $G{=}8$ MATH-500 Transfer Check}
\label{app:mathg8}

We additionally evaluated the matched $G{=}8$ checkpoints on MATH-500.
Unlike the $G{=}4$ single-checkpoint transfer result in the main text, this matched $G{=}8$ transfer check is not positive for Sign: DrGRPO reaches 32.3\% mean accuracy over three seeds, while Sign reaches 29.3\% with high variance.
We therefore do not use it as supporting evidence for cross-dataset generality.

\begin{table}[h]
\centering
\caption{Per-run MATH-500 greedy accuracy for the matched $G{=}8$ transfer check.}
\label{tab:g8math}
\begin{tabular}{lcc}
\toprule
Run & Seed & Accuracy (\%) \\
\midrule
DrGRPO $G{=}8$ & 42 & 32.0 \\
DrGRPO $G{=}8$ & 43 & 32.0 \\
DrGRPO $G{=}8$ & 44 & 32.8 \\
Sign $G{=}8$ & 42 & 43.4 \\
Sign $G{=}8$ & 43 & 21.4 \\
Sign $G{=}8$ & 44 & 23.2 \\
\bottomrule
\end{tabular}
\end{table}